\documentclass{elsarticle}

\usepackage{graphicx,latexsym}
\usepackage{amsmath,amsfonts,amssymb}
\usepackage{times,color}

\usepackage{epsfig,algorithm,algpseudocode}
\usepackage{array,multirow,graphicx}
\usepackage{siunitx}
\usepackage{adjustbox}
\usepackage{booktabs}
\usepackage{subcaption}

\usepackage{tabularx}
\usepackage{amsthm}

\newtheorem{theorem}{Theorem}

\allowdisplaybreaks

\journal{AAAAAA}

\def\beps{\ensuremath{\boldsymbol{\epsilon}}}
\def\bF{\ensuremath{\mathbf{F}}}
\def\bh{\ensuremath{\mathbf{h}}}

\begin{document}

\begin{frontmatter}

\title{Going Deeper with Five-point Stencil Convolutions for Reaction-Diffusion Equations}


\author[a]{Yongho Kim}
\ead{ykim@mpi-magdeburg.mpg.de}

\author[b]{Yongho Choi\corref{mycorrespondingauthor}}
\cortext[mycorrespondingauthor]{Corresponding author}
\ead{yongho\_choi@daegu.ac.kr}
\ead[url]{https://sites.google.com/view/yh-choi}

\address[a]{Faculty of Mathematics, Otto-von-Guericke-Universität Magdeburg, Universitätsplatz 2, 39106, Magdeburg, Germany}
\address[b]{Department of Computer \& Information Engineering, Daegu University, Gyeongsan-si, Gyeongsangbuk-do 38453, Republic of Korea}

\begin{abstract}
Physics-informed neural networks have been widely applied to partial differential equations with great success because the physics-informed loss essentially requires no observations or discretization. However, it is difficult to optimize model parameters, and these parameters must be trained for each distinct initial condition. To overcome these challenges in second-order reaction-diffusion type equations, a possible way is to use five-point stencil convolutional neural networks (FCNNs). FCNNs are trained using two consecutive snapshots, where the time step corresponds to the step size of the given snapshots. Thus, the time evolution of FCNNs depends on the time step, and the time step must satisfy its CFL condition to avoid blow-up solutions. In this work, we propose deep FCNNs that have large receptive fields to predict time evolutions with a time step larger than the threshold of the CFL condition. To evaluate our models, we consider the heat, Fisher's, and Allen--Cahn equations with diverse initial conditions. We demonstrate that deep FCNNs retain certain accuracies, in contrast to FDMs that blow up.
\end{abstract}


\end{frontmatter}

%\tableofcontents

\section{Introduction}
Natural and chemical phenomena, as well as some problems in the real world, can be described through mathematical expressions, in particular partial differential equations (PDEs). For example, fluid flow, chemical reaction-diffusion, phase separation, image analysis, image segmentation, cell division, the spread of infectious diseases, etc., can be mathematically expressed. Therefore, to find solutions to PDEs or mathematically analyze the characteristics of phenomena and changes in energy over time that these equations represent, numerical studies aimed at approximating PDE solutions are continuously being conducted such as finite difference method (FDM) \cite{RLV2007,YZ2009,EBKP2011,JKDJSYYC2017,YLJK2017}, finite element method (FEM) \cite{NFJXYY2011,AD2011,CJ2012,NKSMAB2022}, finite volume method (FVM) \cite{QXXXNF2012,EFA2015,SM2015,PZZL2022}, and so on.

% 최근 ML, DL을 이용한 PDE의 해를 찾는 연구도 다양하게 진행되고 있음.
% CNN사용 FDM 연구(수치연구) 다양하게 진행되고 있음 언급(히스토리 정리)
Moreover, machine learning/deep learning models have been developed to solve approximately PDE solutions. The application of physics-informed neural networks (PINNs) \cite{pinn} has led to great success in solving partial differential equations because the physics-informed loss does not require any observations or discretization. Also, PINNs achieve acceptable accuracy for diverse simulations \cite{pinn2,pinn3,pinn4,pinn5,pinn6}. However, optimizing model parameters remains a challenge, and PINNs should be trained separately by each initial condition. 

To address these problems, a possible way is to use data-driven models that can learn numerical schemes using snapshots and predict solutions at further time steps. Using a supervised learning approach, convolutional neural networks \cite{cnn} have been widely applied to solve partial differential equations \cite{phygeo,convpde,poicnn} because the mechanism of the convolution operator is similar to numerical methods that utilize neighboring points to obtain values at the next time step. However, the prediction of data-driven machine learning models and numerical methods is strongly affected by the time step of given snapshots, so appropriate time steps should be chosen.

%%% 수정 할 것
Here, we focus on the \emph{receptive field} \cite{receptivefield,yolov4} which refers to the size of input nodes that affect a single output node.
Modern convolutional neural networks \cite{vgg, resnet, effnet} have been designed to acquire large receptive fields for good feature extraction. In other words, the large receptive field increases the capacity of the indirect connectivity between an input and its output so that plenty of the input nodes are involved in the output extraction. In FDMs, simulation errors and time steps are also influenced by the receptive field size related to the order of approximations to derivatives (e.g., 5-point stencil vs. 9-point stencil). 

% 폰노이만 아날리시스->CFL condition 통한 열방정식 dt에 대한 안정 범위 유도 기술(2D)
\begin{theorem}[Stability analysis]\label{stability}
The stability condition of 2D heat equation ($\phi(x,y)_t=\phi_{xx}+\phi_{yy}$) is $h^2 /4$, where $h=\frac{1}{\Delta x}=\frac{1}{\Delta y}$.
\end{theorem}
\begin{proof}
\begin{align}
&\frac{\phi_{ij}^{n+1}-\phi_{ij}^n}{\Delta t} = \frac{\phi_{i+1,j}^{n}-2\phi_{ij}^n+\phi_{i-1,j}^{n}}{(\Delta x)^2}+\frac{\phi_{i,j+1}^{n}-2\phi_{ij}^n+\phi_{i,j-1}^{n}}{(\Delta y)^2}. \nonumber
\end{align}
$\phi(x,y,t_n)=e^{iqx}e^{iry}$ then,
\begin{align}
&\frac{Gf-1}{\Delta t}=\frac{e^{iq\Delta x}+e^{-iq\Delta x}-2}{(\Delta x)^2}+\frac{e^{ir\Delta y}+e^{-ir\Delta y}-2}{(\Delta y)^2}, \nonumber
\end{align}
where $Gf$ is a growth factor defined as
\begin{align}
&Gf=1-2 \frac{\Delta t}{(\Delta x)^2}(1-\cos(q\Delta x))-2\frac{\Delta t}{(\Delta y)^2}(1-\cos(r\Delta y)). \nonumber
\end{align}
The worst case is $q\Delta x=r\Delta y = \pi$, then 
\begin{align}
&Gf=1-4\frac{\Delta t}{(\Delta x)^2}-4\frac{\Delta t}{(\Delta y)^2}. \nonumber
\end{align}
Therefore, the stability condition is
\begin{align}
&\frac{\Delta t}{(\Delta x)^2}+\frac{\Delta t}{(\Delta y)^2} \leq \frac{1}{2}. \nonumber
\end{align}
Since $h=\frac{1}{\Delta x}=\frac{1}{\Delta y}$, we obtain
\begin{align}
&\Delta t \leq \frac{h^2}{4}. \label{heat2d_stab}
\end{align}
\end{proof}

Basically, the time step size $\Delta t$ can be decided by Theorem  \ref{stability}, which provides an analysis of the stability range of the explicit scheme for the two-dimensional heat equation. The stability analysis determines the range of suitable time steps that ensures the numerical solution remains stable. In other words, if $\Delta t$ does not satisfy the Eq. \eqref{heat2d_stab}, a blowup could occur. 

In this paper, the main idea of our approach is to utilize a receptive field that permits a time step larger than the threshold of the CFL condition. Therefore, we propose a deep CNN architecture to increase the receptive field size. 
%We propose a deep CNN architecture that can produce stable results even when using $\Delta t$ beyond the threshold of the stability condition \eqref{heat2d_stab}. 
The contents of this paper is as follows: In Section \ref{sec:methods}, we explain our proposed deep five-point stencil convolutional neural networks (deep FCNNs) and algorithms. In Section \ref{sec:results}, we perform numerical simulations for various initial conditions. In Section \ref{sec:discu}, we summarize the paper and discuss a possible research direction.

\section{Methods and numerical solutions} \label{sec:methods}

\begin{figure}[t]
\centering
\includegraphics[width=0.8\columnwidth]{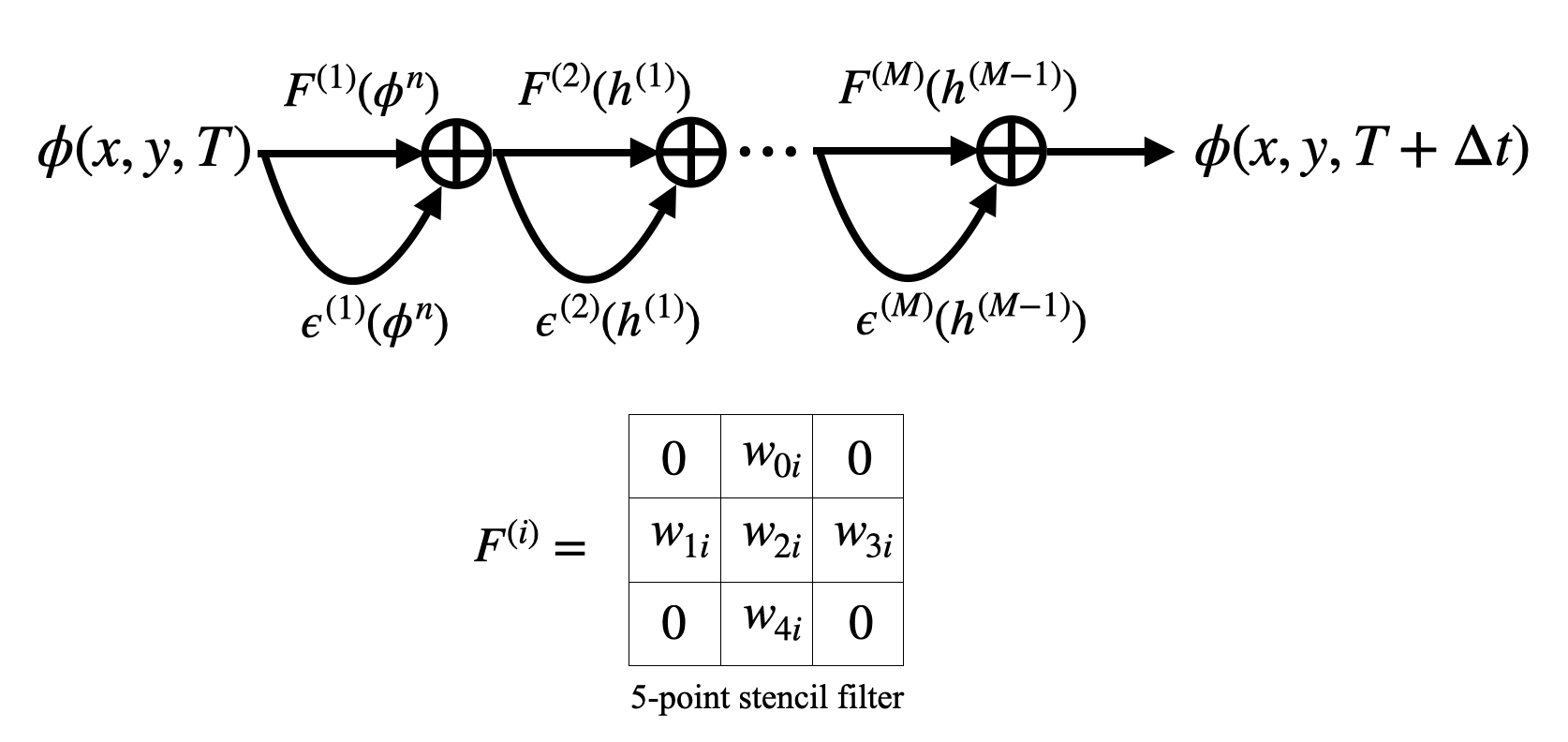}
\caption{A deep five-stencil convolutional neural network with 5-point stencil filters $\bF^{(i)}(x)$ and $r$-th order polynomial functions $\beps^{(i)}(x)=a_{i}+\sum^{r}_{k=1}a_{ki}x^k$ where $a_i, a_{ki} \in \mathbb{R}$ for all $i, k$.}%
\label{fig:fcnn}
\end{figure}

\begin{figure}[t]
\centering
\includegraphics[width=0.8\columnwidth]{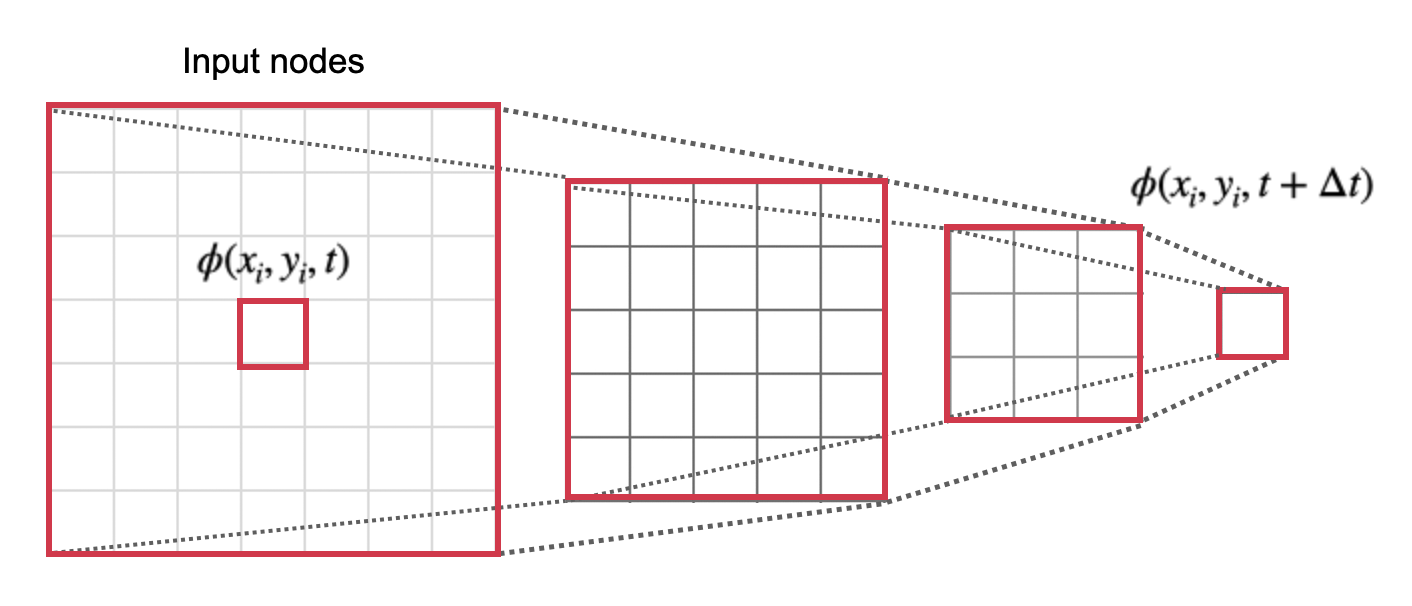}
\caption{The receptive field size of three five-point stencil convolutional layers is $7 \times 7$.}%
\label{fig:receptive}
\end{figure}

\emph{Five-point stencil convolutional neural networks} (FCNNs) \cite{fcnn} learn the explicit finite difference methods of second-order reaction-diffusion equations using only two consecutive snapshots $\phi_0$ and $\phi_1$. Even in the case that $\phi_1$ is a contaminated snapshot (injected Gaussian random noise), FCNNs are trainable, and a pretrained model can predict the time evolution of any initial conditions without requiring additional training for each initial condition. However, explicit FCNNs have a limitation in the selection of the time step $\Delta t$ as stated in Theorem \ref{stability} like explicit FDMs.

We propose \emph{deep five-point stencil convolutional neural networks} (deep FCNNs) consisting of multiple five-stencil layers to prevent the performance degradation of numerical simulations for second-order reaction-diffusion partial differential equations depending on time step sizes. A five-stencil layer has a five-stencil operator $\bF$ and a trainable $r$-th order polynomial function $\beps$ so
the deep FCNN can be described as follows:
\begin{subequations}\label{eq:deepfcnn}
\begin{align*}
&\bh^{(1)} = \bF^{(1)}(\phi(x,y,T))+\beps^{(1)}(\phi(x,y,T)) \\
&\bh^{(m)} = \bF^{(m)}(\bh^{(m-1)})+\beps^{(m)}(\bh^{(m-1)}),\,\quad m=2,\cdots, M-1, \\
&\phi(x,y,T+\Delta t) = \bF^{(M)}(\bh^{(M-1)} )+\beps^{(M)}(\bh^{(M-1)} ) \\
\end{align*}
\end{subequations}
where $\phi(x,y,T)$ is the solution at $t=T$ and $\phi(x,y,T+\Delta t)$ is the solution at $t=T+\Delta t$ as shown in Figure \ref{fig:fcnn}. The main objective of deep FCNNs is to increase the receptive field size. For instance, networks that consist of a single five-stencil operation, such as FCNNs and second-order finite difference methods, have a receptive field of $3 \times 3$. With $M$ layers, the receptive field size is $(2M+1) \times (2M+1)$ (e.g., Figure \ref{fig:receptive}). 

In the training session, we utilize two nonconsecutive snapshots $\phi_0$ and $\phi_k$ to train a deep FCNN that predicts evolutions with a larger time step $\Delta t_L$ than the time step size $\Delta t_s$ of the provided snapshots as illustrated in Algorithm \ref{alg:fcnn}. However, the design of the receptive field should be carefully considered due to potential issues with optimization problems and inference slowdowns resulting from a large number of model parameters. Therefore, it is important to select an appropriate $\phi_k$ satisfying $\Delta t_s \leq \frac{h^2}{4\alpha} \leq \Delta t_L$.

\begin{algorithm}[t]
\caption{Training a deep FCNN with two nonconsecutive snapshots $\phi_0$ and  $\phi_k$}\label{alg:fcnn}
\begin{algorithmic}
\Require $\phi_0$, $\phi_k$, $k \in \mathbb{Z}$, $\epsilon > 0$
\While{$l > \epsilon$}
\State $l(\theta) \gets \parallel \phi_k - f(\phi_0;\theta)\parallel_2^2$
\State Update $\theta$ based on the gradient of $l$
\EndWhile
\end{algorithmic}
\end{algorithm}

\section{Simulation results} \label{sec:results}
\begin{table}[t]
\caption{Relative $L_2$ errors: (HE) Heat, (FE) Fisher's, and (AC) Allen--Cahn equations}
\begin{adjustbox}{width=320pt,center}\label{table:relerr}
\begin{tabular}{*8c}
\toprule
%{} &  \multicolumn{4}{c}{Initial condition}\\
& Initial condition & Sierra   & Star    & Circle   & Torus & Maze & Cells\\
\multirow{4}{*}{HE} &{$t=$}   & 0.006   & 0.006    & 0.006  & 0.006 & 0.006 & 0.006 \\ % HE
& FCNN($\Delta t_L$)  & $\mathbf{8.207\times 10^{-4}}$ & $\mathbf{2.693\times 10^{-4}}$   &  $\mathbf{2.891\times 10^{-4}}$ &  $\mathbf{5.630\times 10^{-4}}$ & $\mathbf{2.700\times 10^{-3}}$ &$\mathbf{5.897\times 10^{-4}}$  \\ 
& FDM($\Delta t_s$)  & $8.561\times 10^{-4}$ & $2.747\times 10^{-4}$  & $2.959\times 10^{-4}$  & $5.734\times 10^{-4}$ & $2.704\times 10^{-3}$ &$5.968\times 10^{-4}$\\
%& FDM($3\Delta t$) &  -  &  $3.59\times 10^{-4}$  & - & - & -\\

\midrule
\multirow{4}{*}{FE} & {$t=$}   & 0.006   & 0.006    & 0.006  & 0.003 & 0.006 & 0.006\\
& FCNN($\Delta t_L$)  &  $\mathbf{2.745\times 10^{-3}}$ &$\mathbf{3.306\times 10^{-2}}$  & $\mathbf{3.581\times 10^{-2}}$  & $\mathbf{2.457\times 10^{-3}}$ & $\mathbf{1.561\times 10^{-2}}$&$\mathbf{3.205\times 10^{-2}}$\\
& FDM($\Delta t_s$)  &  $3.188\times 10^{-3}$ & $3.423\times 10^{-2}$   &$3.703\times 10^{-2}$ & $2.675\times 10^{-3}$ & $1.646\times 10^{-2}$&$3.321\times 10^{-2}$\\
%& FDM($3\Delta t$) &  $3.85\times 10^{-3}$  &  $7.38\times 10^{-3}$  & $7.44\times 10^{-3}$ & $6.42\times 10^{-3}$ & $2.42\times 10^{-3}$\\

\midrule
\multirow{4}{*}{AC} & {$t=$}   & 0.006   & 0.006    & 0.006  & 0.006 & 0.006& 0.006\\ % AC
& FCNN($\Delta t_L$)  &  $\mathbf{2.618\times 10^{-2}}$ & $\mathbf{4.812\times 10^{-4}}$  & $4.296\times 10^{-5}$  &$\mathbf{3.801\times 10^{-5}}$ & $1.144\times 10^{-3}$&$6.766\times 10^{-4}$\\
& FDM($\Delta t_s$)  &  $5.419\times 10^{-2}$ & $4.812\times 10^{-4}$   &$4.296\times 10^{-5}$ & $3.802\times 10^{-5}$ & $1.144\times 10^{-3}$& $6.766\times 10^{-4}$\\
%& FDM($3\Delta t$) &  $9.65\times 10^{-5}$  &  $2.98\times 10^{-6}$  & $2.98\times 10^{-6}$ & $5.42\times 10^{-6}$ & $7.77\times 10^{-5}$\\
\bottomrule
\end{tabular}
\end{adjustbox}
\end{table}

Reaction-diffusion equations are commonly used to various phenomena such as pattern formation \cite{AMVT1991,LGSPSC2021,RAVG2021}, bacterial branching growth \cite{IGYKICEB1998}, epidemic model \cite{HMXAZTLZ2018,XLZY2022}, traffic flow \cite{MF2000}, and so on. We consider the heat equation (HE), Fisher's equation (FE), and Allen--Cahn equation (AC). The governing equations are as follows:
\begin{itemize}
\item Heat equation(HE): \begin{align}
\phi_t&=\alpha\bigtriangleup\phi. \label{heateq}
%\phi^{n+1}&=conv(\phi^{n})+\phi^{n}. \label{heateqdis}
\end{align}
\item Fisher's equation(FE): \begin{align}
\phi_t&=\alpha\bigtriangleup\phi + \beta(\phi-\phi^2). \label{fseq}
%\phi^{n+1}&=conv(\phi^{n})+\phi^{n} + \Delta t\beta(\phi^n-(\phi^n)^2). \label{feateqdis}
\end{align}
\item Allen--Cahn equation(AC): \begin{align}
\phi_t&=\alpha\bigtriangleup\phi+ \beta(\phi-\phi^3). \label{aceq}
%\phi^{n+1}&=conv(\phi^{n})+\phi^{n} + \Delta t\beta(\phi^n-(\phi^n)^3). \label{acateqdis}
\end{align}
\end{itemize}
where $\alpha$ is a diffusion coefficient and $\beta$ is a reaction coefficient.

% 유한차분법(수치기법) 이용 governing equation 푸는 방법 기술[바운더리 노이만]
In the FDM sense, the Eqs. \eqref{heateq}-\eqref{aceq} can be discretized as Eqs. \eqref{heateqdis}-\eqref{acateqdis} respectively.
\begin{align}
\phi_{ij}^{n+1}&=\phi_{ij}^{n} + \Delta t \alpha \bigtriangleup_h \phi_{ij}^{n}, \label{heateqdis}\\
\phi_{ij}^{n+1}&=\phi_{ij}^{n} + \Delta t(\alpha \bigtriangleup_h \phi_{ij}^{n} + \beta(\phi_{ij}^n-(\phi_{ij}^n)^2)), \label{feateqdis}\\
\phi_{ij}^{n+1}&=\phi_{ij}^{n} + \Delta t(\alpha \bigtriangleup_h \phi_{ij}^{n} + \beta(\phi_{ij}^n-(\phi_{ij}^n)^3)). \label{acateqdis}
\end{align}
The $\phi_{ij}^n$ is the approximation of $\phi(x_i,y_j,n\Delta t)$ and $\Delta t$ is a time step size. The Laplacian $ \bigtriangleup_h$ can be numerically computed as follows:
\begin{equation}
\bigtriangleup_h \phi_{ij}=\frac{\phi_{i+1,j}+\phi_{i-1,j}+\phi_{i,j+1}+\phi_{i,j-1}-4\phi_{ij}}{h^2}, \quad 1\leq i \leq N_x, 1\leq j \leq N_y \nonumber
\end{equation}
on the computational domain $\Omega=(a,b)\times(c,d)$ with uniform mesh size($h=(b-a)/N_x=(d-c)/N_y: N_x$ and $N_y$ are the number of grid points.)

We use the zero Neumann boundary condition: 
\begin{align}
\phi_{0j}&=\phi_{1j}, ~~\phi_{N_xj}=\phi_{N_{x-1}j}, \quad \textrm{for}~~ j=1,2,\ldots,N_y, \nonumber \\ 
\phi_{i0}&=\phi_{i1}, ~~\phi_{iN_y}=\phi_{iN_{y-1}}, ~~\quad \textrm{for}~~ i=1,2,\ldots,N_x. \nonumber
\end{align}

Initially, we set the diffusion and reaction coefficients of the HE, FE, and AC equations to $(\alpha,\beta)=(1,0),(1,100),(1,6944)$, respectively and define the domain as $100\times 100$ rectangle grids on $\Omega=(0,1)\times (0,1)$. Thus, the threshold of the time step $\frac{h^2}{4\alpha}$ is $2.5\times 10^{-5}$. Here, we use $\Delta t_s=2\times 10^{-5}$, which is smaller than the threshold, and $\Delta t_{L}=6\times 10^{-5}$, which is larger than the threshold, to simulate FDMs and FCNNs. the training snapshots are generated by the discretized Eqs. \eqref{heateqdis}-\eqref{acateqdis}. Also, we build deep FCNNs consisting of three layers with 0th, 2nd, and 3rd polynomial functions for the HE, FE, and AC equations respectively. To train the models, the initial conditions are generated by a uniform distribution on the interval $[ -1,1]$.
% FDM 데이터 생성하는 법

For the error metric, the relative $L_2$ error is defined as
\begin{equation*}\label{eq:rel2}
\frac{\parallel \phi - \phi_{ref} \parallel_2}{\parallel \phi_{ref} \parallel_2}.
\end{equation*}
where the references ($\phi_{ref}$) are the solutions obtained by the explicit finite difference method with the time step $\frac{\Delta t_s}{100}$.
We define six initial shapes as follows:
\begin{itemize}
\item Sierra: \begin{align*}
\phi(x,y,0)=\cos(2\pi x)\cos(2\pi y)%\label{heateq}
\end{align*}
\item Star: \begin{align*}
\phi(x,y,0)= \tanh\left(\frac{0.25+0.1\cos(6\theta)-C}{\sqrt{2}\rho}\right)
\end{align*}
\item Circle: \begin{align*}
\phi(x,y,0)= \tanh\left(\frac{0.25-C}{\sqrt{2}\rho}\right)
\end{align*}
\item Torus: \begin{align*}
\phi(x,y,0)= -1+\tanh\left(\frac{0.4-C}{\sqrt{2}\rho}\right)-\tanh\left(\frac{0.3-C}{\sqrt{2}\rho}\right)
\end{align*}
\item Maze: \begin{align*}
\phi(x,y,0) \text{ is manually constructed. (see our code)}
\end{align*}
\item Cells: \begin{align*}
\phi(x,y,0) \text{ is described by three circles using the circle function.} 
\end{align*}
\end{itemize}
where $C=\sqrt{(x-0.5)^2+(y-0.5)^2}$, $\rho\approx 0.012$ is the thickness of the transition layer \cite{accnn}, and
$\theta=\tan^{-1}((y-0.5)/(x-0.5))$ if $x>0.5$; otherwise, $\theta=\pi+\tan^{-1}((y-0.5)/(x-0.5))$. The source code is available from https://github.com/kimy-de/deepfcnn.

Table \ref{table:relerr} shows that the deep FCNNs with $\Delta t_L$ are comparable to the FDMs with $\Delta t_s$ in all the results. In contrast, the FDMs blow up when $\Delta t_L$ is used to predict their time evolutions as shown in (d) of Figure \ref{fig:he}, \ref{fig:fe}, and \ref{fig:ac}. 

Figure \ref{fig:he} depicts the time evolution of the six unseen shapes for the heat equation. It is observed that the FDMs with $\Delta t_L$ blow up while the deep FCNNs yield stable solutions in all the cases. Figure \ref{fig:fe}, and \ref{fig:ac} display the time evolution of the six unseen shapes for the FE and AC equations respectively. With $\Delta t_L$, the FDMs blow up, whereas the deep FCNNs predict the evolution comparable to the results of the FDMs with $\Delta t_s$.

%%%%%%%%%%%
\begin{figure}[htbp]
\begin{minipage}{0.99\linewidth}
\centering
\includegraphics[width=4.3in]{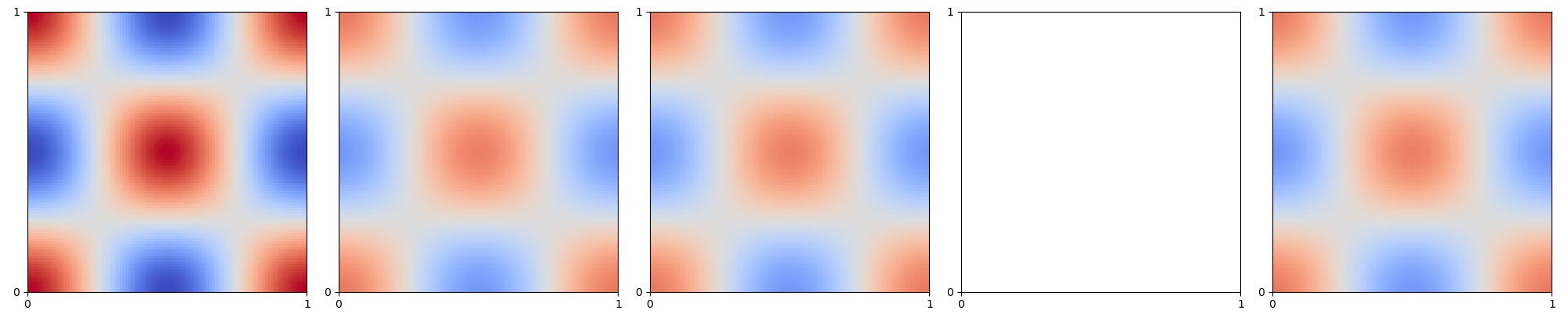}  \\
\end{minipage}\\
\begin{minipage}{0.99\linewidth}
\centering
\includegraphics[width=4.3in]{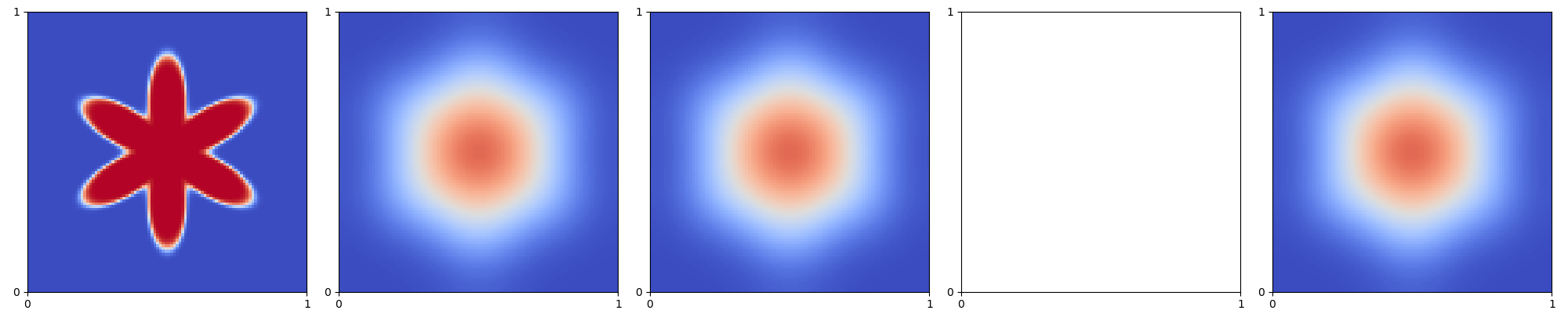} \\
\end{minipage}\\
\begin{minipage}{0.99\linewidth}
\centering
\includegraphics[width=4.3in]{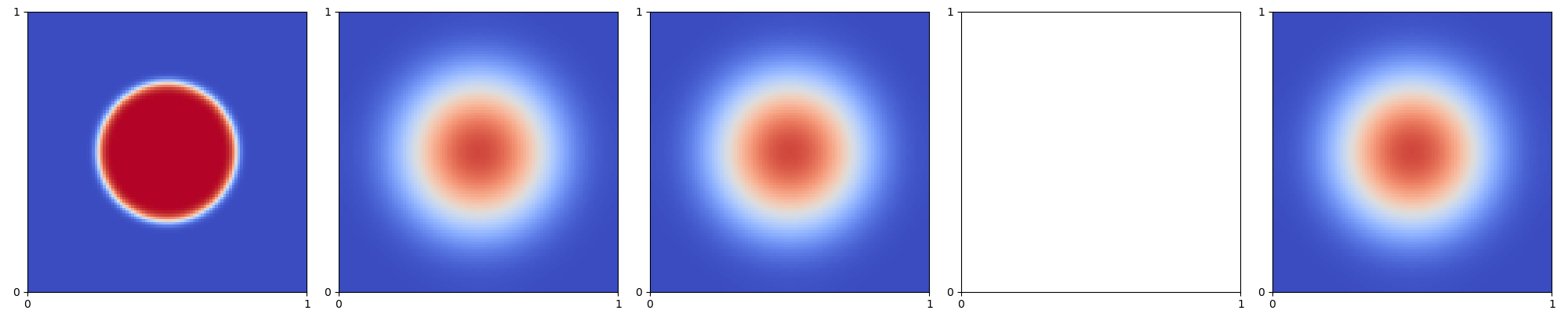} \\
\end{minipage}\\
\begin{minipage}{0.99\linewidth}
\centering
\includegraphics[width=4.3in]{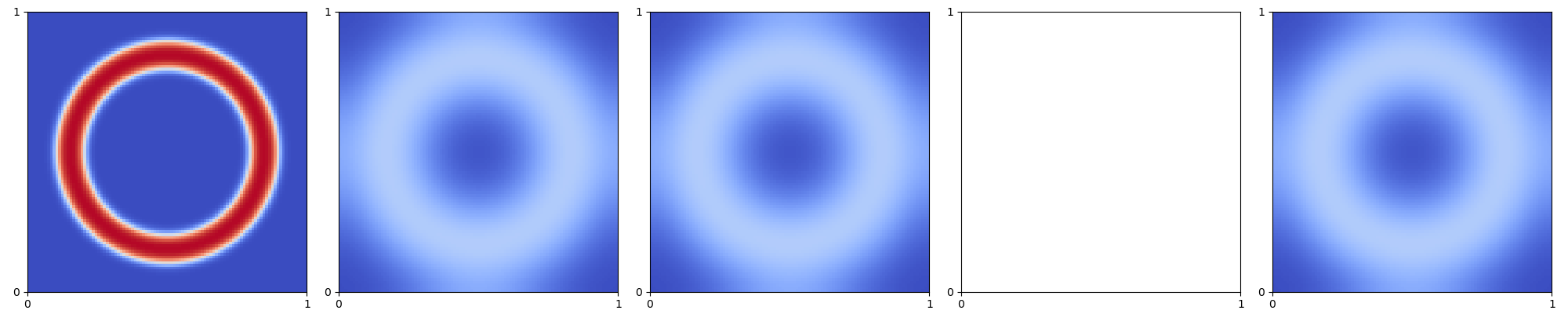} \\
\end{minipage}\\
\begin{minipage}{0.99\linewidth}
\centering
\includegraphics[width=4.3in]{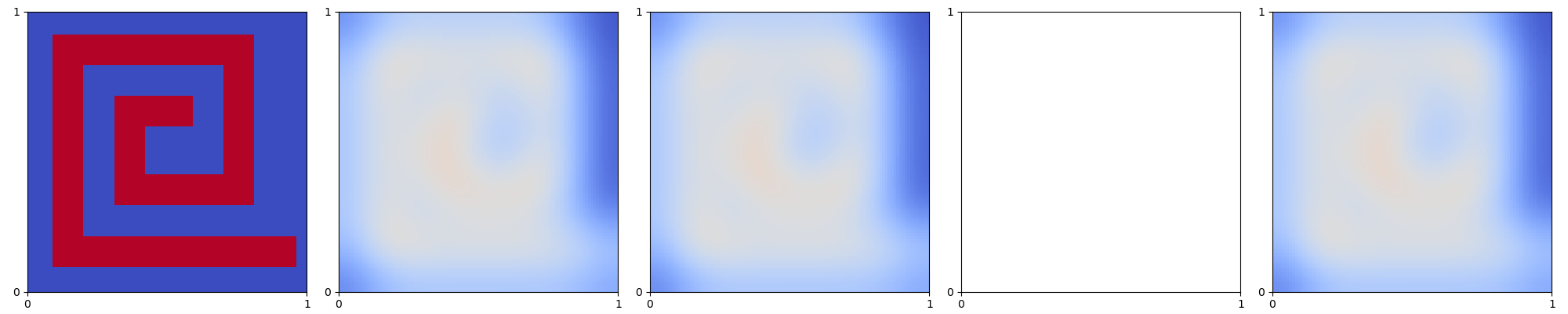} \\
\end{minipage}\\
\begin{minipage}{0.99\linewidth}
\centering
\includegraphics[width=4.3in]{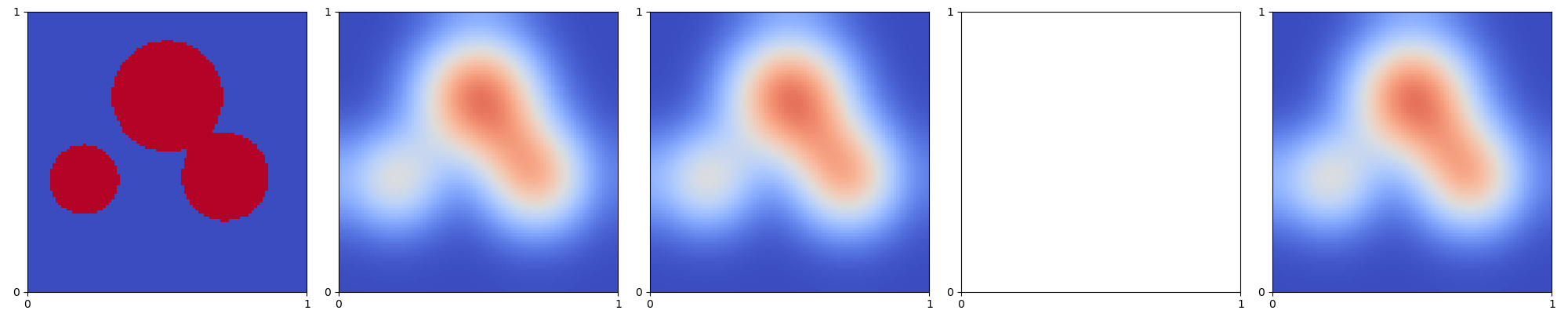} \\
\end{minipage}\\
\begin{minipage}{0.99\linewidth}
\centering
(a) \qquad \qquad (b) ~~ \qquad \qquad (c) \qquad  \qquad \quad (d) \qquad \qquad~~ (e)
\end{minipage}
\caption{Heat equation: (a) initial conditions, (b) reference solutions, (c) FDM results with $\Delta t_s$, (d) FDM results with $\Delta t_L$, and (e) FCNN results with $\Delta t_L$.}
\label{fig:he}
\end{figure}
%%%%%%%%

%%%%%%%%%%%
\begin{figure}[htbp]
\begin{minipage}{0.99\linewidth}
\centering
\includegraphics[width=4.3in]{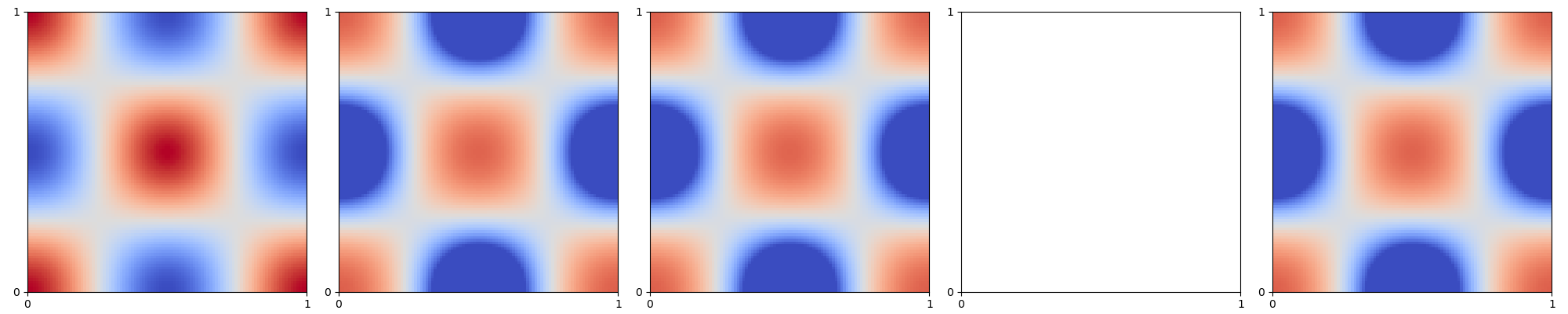}  \\
\end{minipage}\\
\begin{minipage}{0.99\linewidth}
\centering
\includegraphics[width=4.3in]{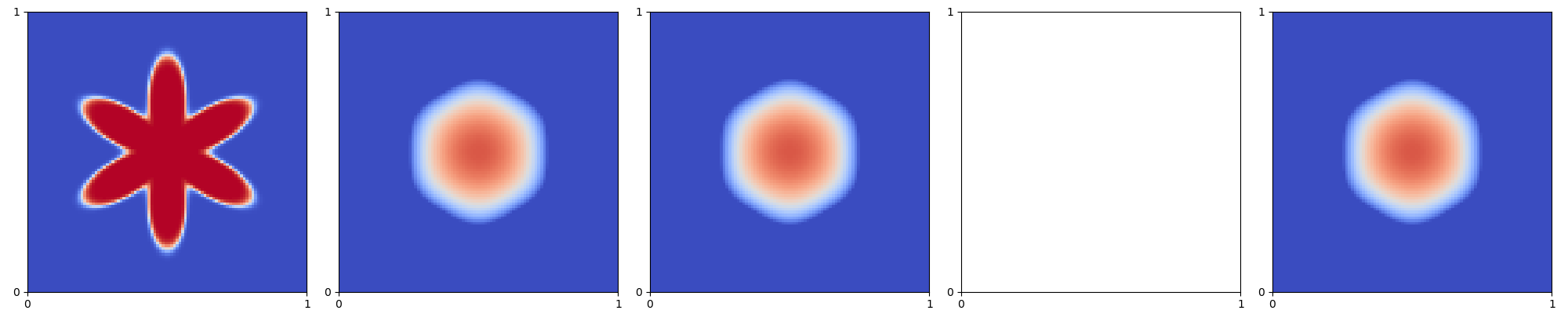} \\
\end{minipage}\\
\begin{minipage}{0.99\linewidth}
\centering
\includegraphics[width=4.3in]{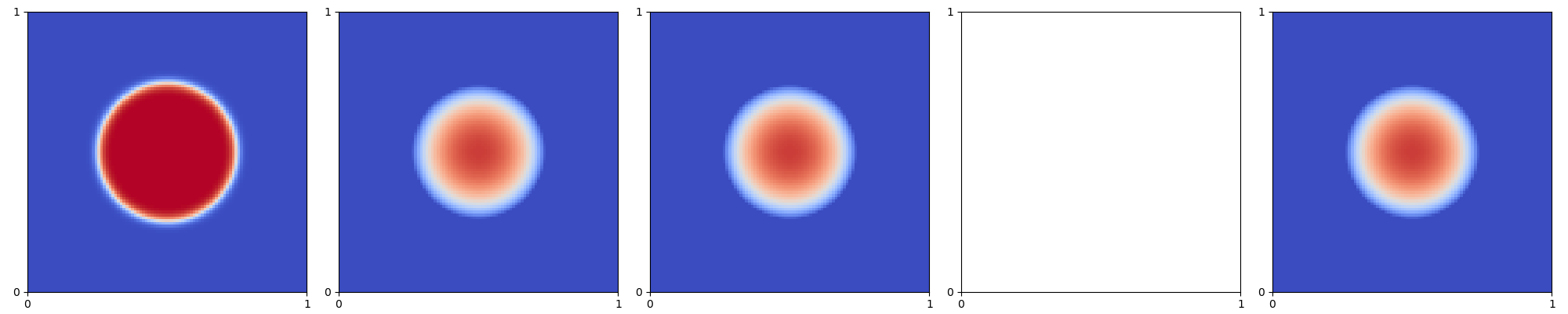} \\
\end{minipage}\\
\begin{minipage}{0.99\linewidth}
\centering
\includegraphics[width=4.3in]{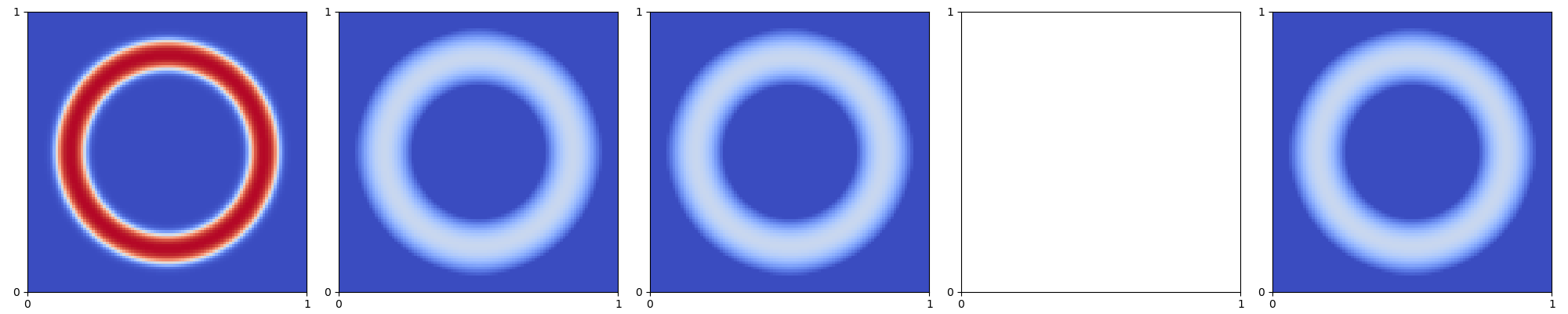} \\
\end{minipage}\\
\begin{minipage}{0.99\linewidth}
\centering
\includegraphics[width=4.3in]{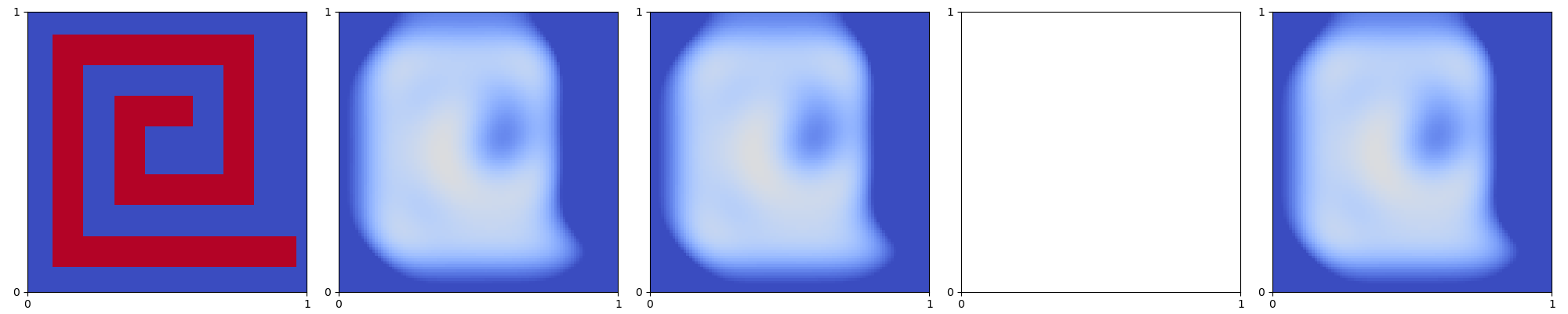} \\
\end{minipage}\\
\begin{minipage}{0.99\linewidth}
\centering
\includegraphics[width=4.3in]{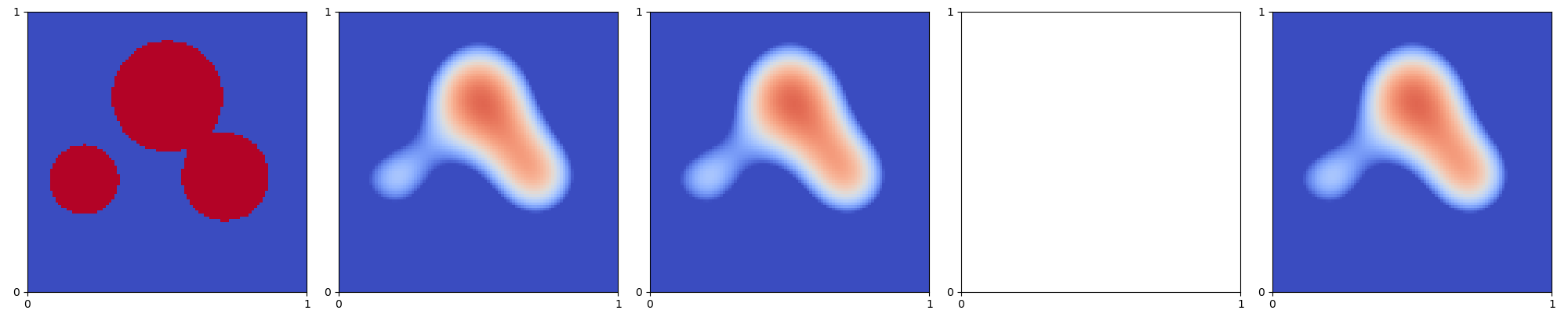} \\
\end{minipage}\\
\begin{minipage}{0.99\linewidth}
\centering
(a) \qquad \qquad (b) ~~ \qquad \qquad (c) \qquad  \qquad \quad (d) \qquad \qquad~~ (e)
\end{minipage}
\caption{Fisher's equation: (a) initial conditions, (b) reference solutions, (c) FDM results with $\Delta t_s$, (d) FDM results with $\Delta t_L$, and (e) FCNN results with $\Delta t_L$.}
\label{fig:fe}
\end{figure}
%%%%%%%%

%%%%%%%%%%%
\begin{figure}[htbp]
\begin{minipage}{0.99\linewidth}
\centering
\includegraphics[width=4.3in]{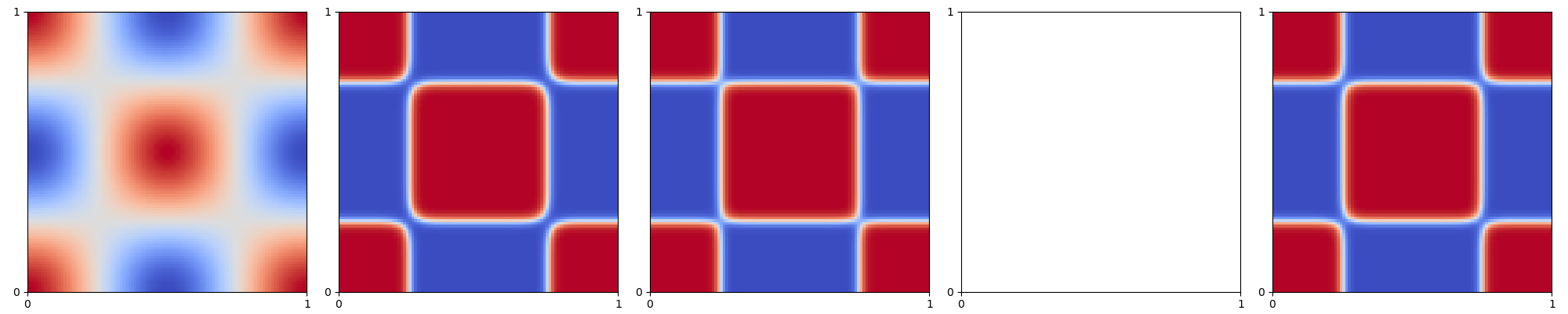}  \\
\end{minipage}\\
\begin{minipage}{0.99\linewidth}
\centering
\includegraphics[width=4.3in]{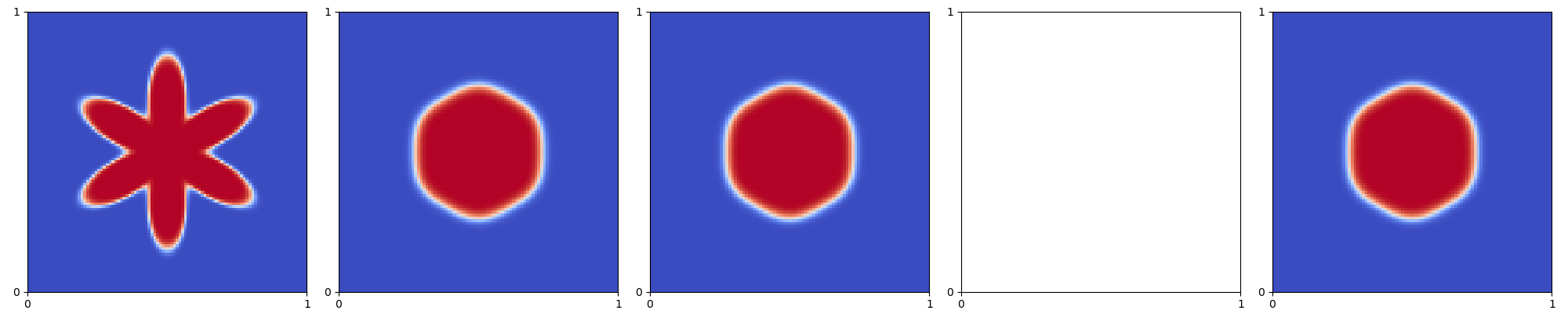} \\
\end{minipage}\\
\begin{minipage}{0.99\linewidth}
\centering
\includegraphics[width=4.3in]{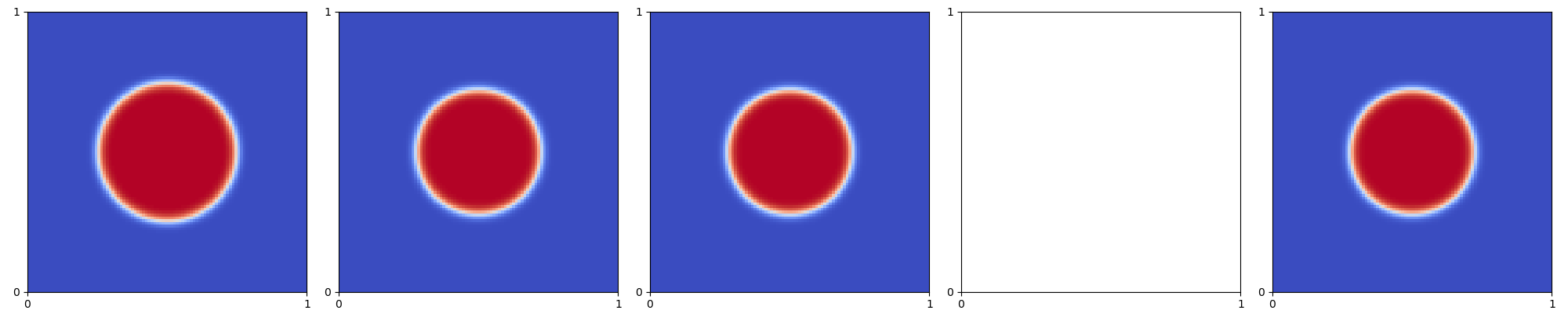} \\
\end{minipage}\\
\begin{minipage}{0.99\linewidth}
\centering
\includegraphics[width=4.3in]{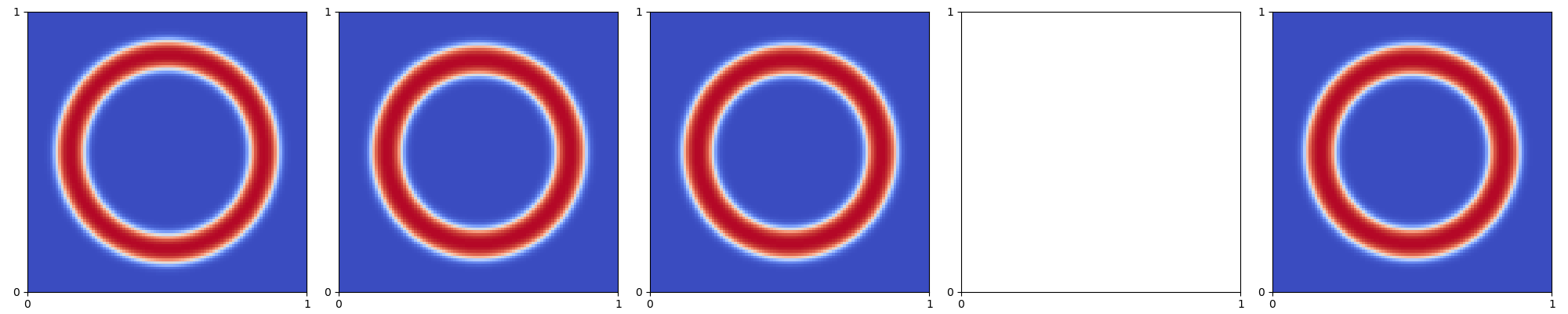} \\
\end{minipage}\\
\begin{minipage}{0.99\linewidth}
\centering
\includegraphics[width=4.3in]{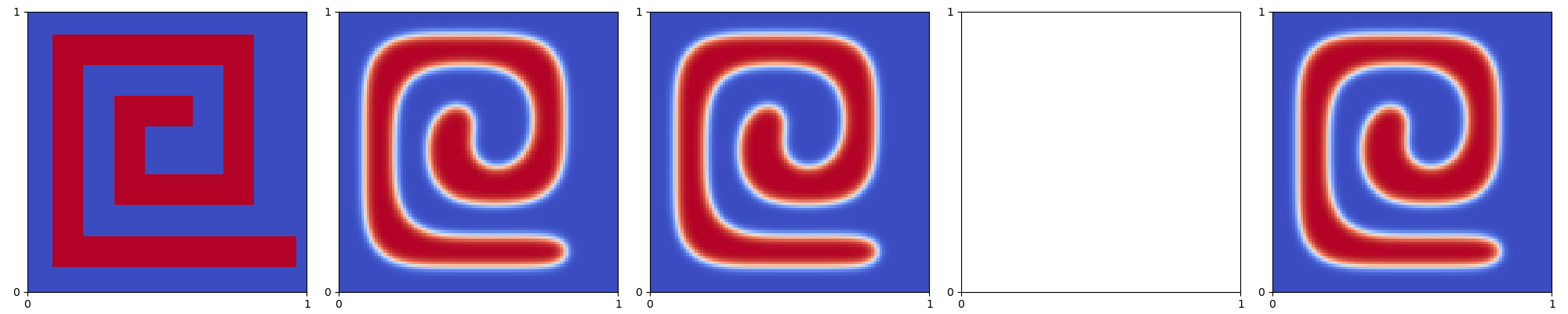} \\
\end{minipage}\\
\begin{minipage}{0.99\linewidth}
\centering
\includegraphics[width=4.3in]{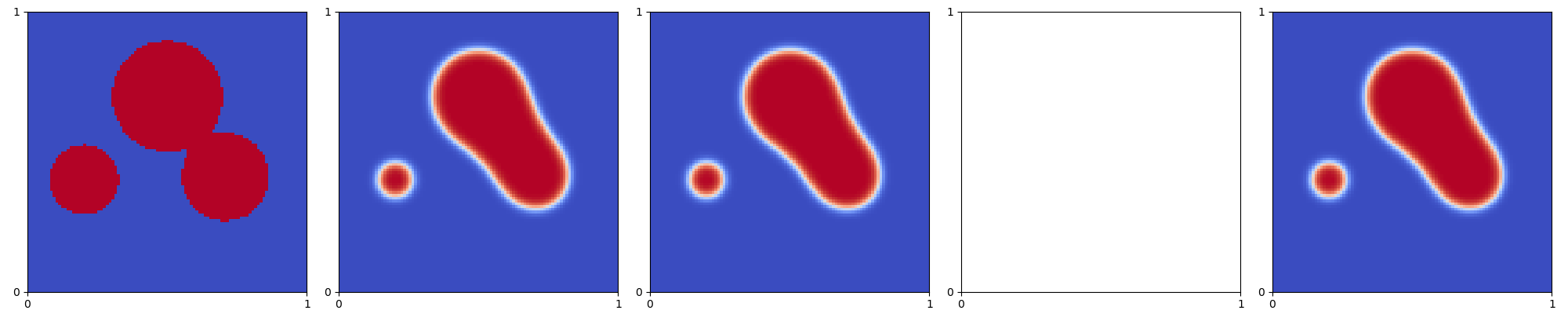} \\
\end{minipage}\\
\begin{minipage}{0.99\linewidth}
\centering
(a) \qquad \qquad (b) ~~ \qquad \qquad (c) \qquad  \qquad \quad (d) \qquad \qquad~~ (e)
\end{minipage}
\caption{Allen-Cahn equation: (a) initial conditions, (b) reference solutions, (c) FDM results with $\Delta t_s$, (d) FDM results with $\Delta t_L$, and (e) FCNN results with $\Delta t_L$.}
\label{fig:ac}
\end{figure}

%\newpage

Furthermore, we simulate the energy dissipation and maximum principle for the AC equation. The AC Eq. \eqref{aceq} follows the energy dissipation law derived from %and which equation derived from
\begin{equation}
\mathcal{E}(\phi)=\int_{\Omega} \left( \frac{F(\phi)}{\epsilon^2} + \frac{1}{2} |\nabla \phi |^2 \right) d\textbf{x}. \label{ACenergyEQ}
\end{equation}
Then $\mathcal{E}(\phi)$ is decreasing in time
\begin{equation}
\frac{\partial}{\partial t}\mathcal{E}(\phi)=-\int_{\Omega} \left| \frac{\partial \phi}{\partial t} \right|^2 d\textbf{x}  \leq 0.
\end{equation}
where $\Omega=(0,1)\times(0,1), N_x=N_y=100, h=1/100$, $\Delta t=\Delta t_L$, and the transition layer thickness is $\epsilon_m=\frac{hm}{2\sqrt{2}tanh^{-1}(0.9)}$ with $m=5$. To check the discrete energy, we rewrite the Eq. \eqref{ACenergyEQ} as $\mathcal{E}^d(\phi)$. Figure \ref{fig_energyDec} shows the surf plots and discrete energy associated with the predictions of the pretrained deep FCNN for the AC equation. The initial condition is defined as $\phi(x,y,0)=rand(x,y)$ which generates random values in the range of $[-1,1]$.

% $\epsilon_m=\frac{hm}{2\sqrt{2}tanh^{-1}(0.9)}$

\begin{figure}[htbp]
\begin{minipage}{0.24\linewidth}
\centering
\includegraphics[clip=true,width=1.3in]{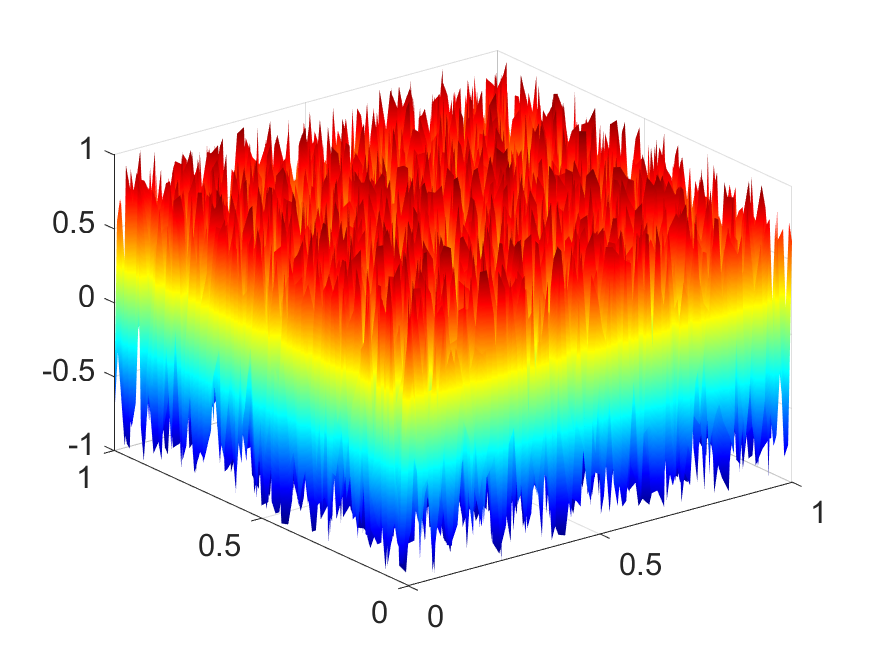}\\
(a) $t=0$
\end{minipage}
\begin{minipage}{0.24\linewidth}
\centering
\includegraphics[clip=true,width=1.3in]{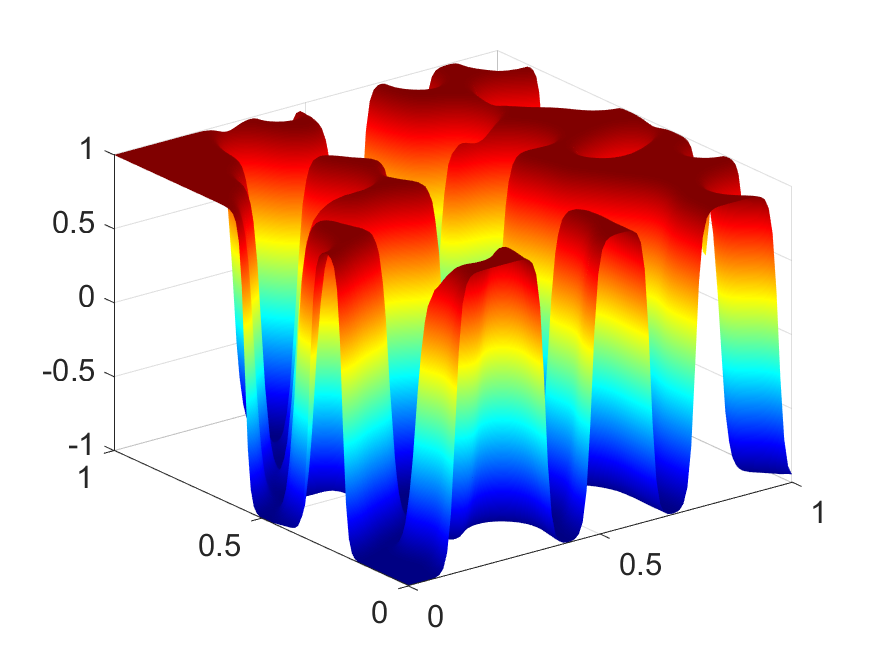}\\
(b) $t=0.0018$
\end{minipage}
\begin{minipage}{0.24\linewidth}
\centering
\includegraphics[clip=true,width=1.3in]{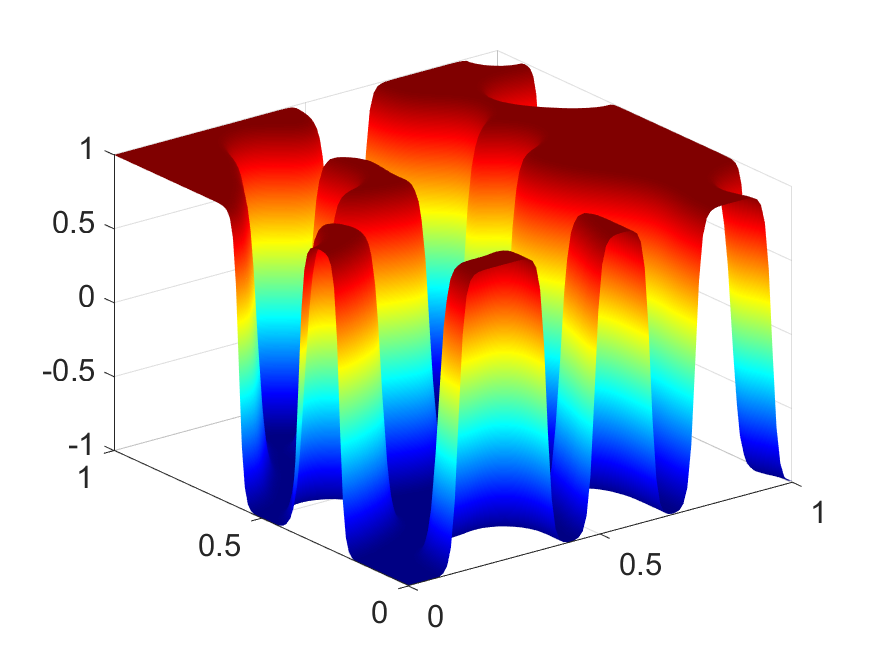}\\
(c) $t=0.003$
\end{minipage}
\begin{minipage}{0.24\linewidth}
\centering
\includegraphics[clip=true,width=1.3in]{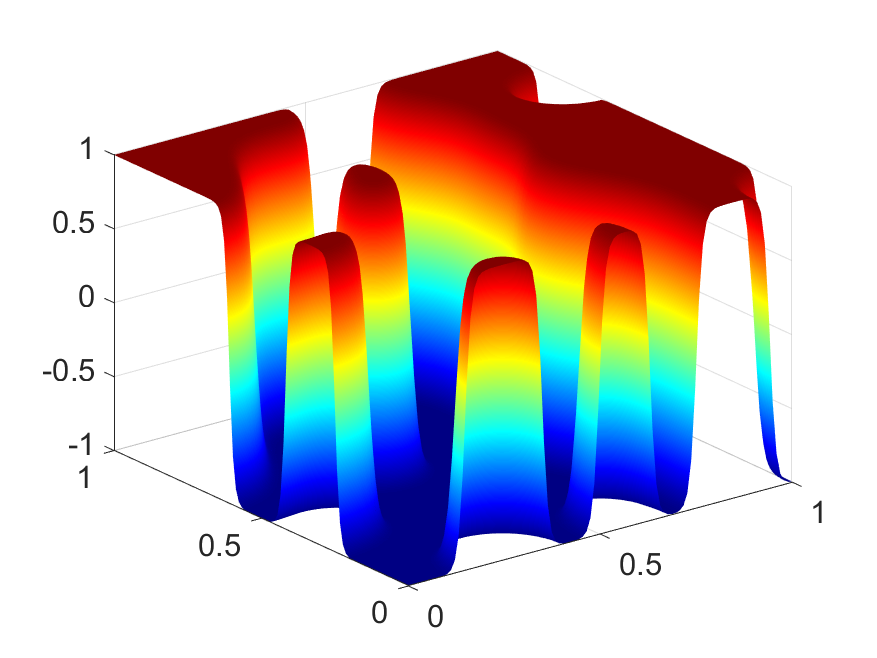}\\ 
(d) $t=0.006$
\end{minipage}\\
\begin{minipage}{0.5\linewidth}
\centering
\includegraphics[clip=true,width=2.4in]{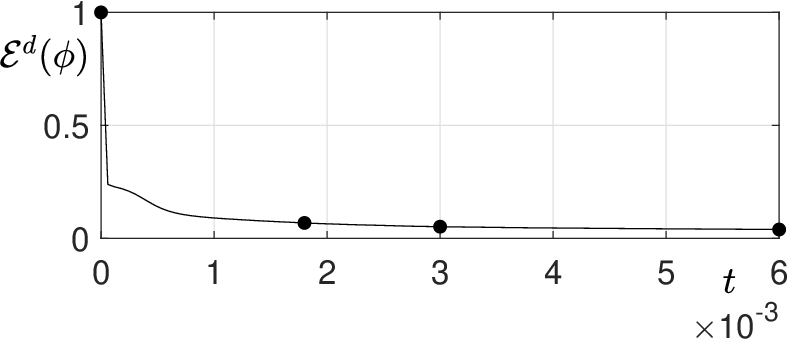}\\
(e)
\end{minipage}
\begin{minipage}{0.45\linewidth}
\centering
\includegraphics[clip=true,width=2.1in]{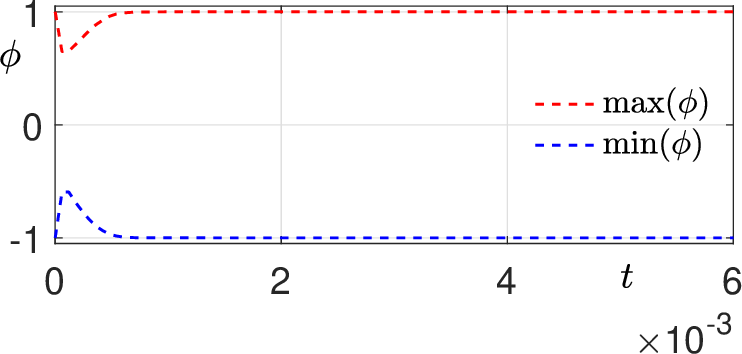}\\
(f)
\end{minipage}
\caption{(a)-(d) Surf plots at the marked points from the left to the right of the energy graph (e). (e) Time dependent normalized discrete total energy $\mathcal{E}^d(\phi(t))/\mathcal{E}^d(\phi(0))$, and (f) Maximum and minimum values of $\phi$ over time evolution.} \label{fig_energyDec} 
\end{figure}

\section{Conclusions} \label{sec:discu}
The use of multiple layers creates large receptive fields capable of accommodating many features that impact each output node.  We proposed deep five-point stencil convolutional neural networks that allow the application of time steps larger than the threshold introduced in the stability analysis. We demonstrated that while the FDMs blow up when their time steps exceed the threshold, deep FCNNs accurately predict the time evolution of six different initial conditions for the heat, Fisher's, and Allen-Cahn equations. Additionally, we showed that deep FCNNs can be trained using only two nonconsecutive snapshots. 

In future work, our research directions could explore the feasibility of deep FCNNs with physics-informed losses that do not require any observations during training. Another approach could involve focusing on the structural re-parameterization of deep FCNNs to improve computational efficiency, as the deep architecture may lead to a reduction in inference speed. Given that the five-point stencil convolution is linear and the composition of the convolutions is also linear, the deep FCNN designed for the heat equation can be structurally re-parameterized into a single-layer perceptron which is faster and retains accuracy. On the other hand, the structural re-parameterization of other equations poses a challenge.

\section*{Acknowledgment}
The corresponding author (Y. Choi) was supported by Basic Science Research Program through the National Research Foundation of Korea(NRF) funded by the Ministry of Education(2022R1I1A307282411).

\end{document}